\documentclass{article} % For LaTeX2e
\usepackage{iclr2021_conference,times}

\usepackage{hyperref}
\usepackage{url}
\usepackage{booktabs}       % professional-quality tables
\usepackage{amsfonts}       % blackboard math symbols
\usepackage{nicefrac}       % compact symbols for 1/2, etc.
\usepackage{microtype}      % microtypography
\usepackage{graphicx}

\usepackage{amsmath}
\usepackage{amsthm}
\usepackage{amssymb}
\usepackage{bm}
\usepackage{mathtools}
\usepackage{xcolor}
\usepackage{wrapfig}
\usepackage{subcaption}
\usepackage{algorithm}
\usepackage{algpseudocode}
\usepackage{xcolor}

\title{Rethinking the Role of Gradient-based Attribution Methods for Model Interpretability}

\newcommand{\latin}[1]{{\it #1}}

\author{Suraj Srinivas \\Idiap Research Institute \& EPFL
\\ \texttt{suraj.srinivas@idiap.ch}
\And Fran\c{c}ois Fleuret \\ University of Geneva
\\ \texttt{francois.fleuret@unige.ch} }

\newtheorem*{observation}{Observation}

\newtheorem*{hypothesis}{Hypothesis}

\newcommand{\R}{\mathbb{R}}
\newcommand{\E}{\mathbb{E}}
\newcommand{\f}[1]{\boldsymbol{#1}}
\newcommand{\X}{\mathbf{x}}
\newcommand{\grad}{\nabla_{\X}}

\iclrfinalcopy % Uncomment for camera-ready version, but NOT for submission.
\begin{document}

\maketitle

\begin{abstract}
Current methods for the interpretability of discriminative deep neural networks commonly rely on
the model's input-gradients, i.e., the gradients of the output logits w.r.t. the inputs. The
common assumption is that these input-gradients contain information regarding $p_{\theta} ( y
\mid \X )$, the model's discriminative capabilities, thus justifying their use for
interpretability. However, in this work we show that these input-gradients can be arbitrarily
manipulated as a consequence of the shift-invariance of softmax without changing the
discriminative function. This leaves an open question: if input-gradients can be arbitrary, why
are they highly structured and explanatory in standard models?

We investigate this by re-interpreting the logits of standard softmax-based classifiers as unnormalized
log-densities of the data distribution and show that input-gradients can be viewed as gradients
of a class-conditional density model $p_{\theta}(\X \mid y)$ implicit within the discriminative
model. This leads us to hypothesize that the highly structured and explanatory nature of
input-gradients may be due to the alignment of this class-conditional model $p_{\theta}(\X \mid
y)$ with that of the ground truth data distribution $p_{\text{data}} (\X \mid y)$. We test this
hypothesis by studying the effect of density alignment on gradient explanations. To achieve this
density alignment, we use an algorithm called score-matching, and propose novel approximations to
this algorithm to enable training large-scale models.

Our experiments show that improving the alignment of the implicit density model with the data
distribution enhances gradient structure and explanatory power while reducing this alignment has
the opposite effect. This also leads us to conjecture that unintended density alignment in
standard neural network training may explain the highly structured nature of input-gradients
observed in practice. Overall, our finding that input-gradients capture information regarding an
implicit generative model implies that we need to re-think their use for interpreting
discriminative models.

\end{abstract}

\section{Introduction}

Input-gradients, or gradients of outputs w.r.t. inputs, are commonly used for the interpretation
of deep neural networks \citep{simonyan2013deep}. For image classification tasks, an input pixel
with a larger input-gradient magnitude is attributed a higher `importance' value, and the resulting
maps are observed to agree with human intuition regarding which input pixels are important for
the task at hand \citep{adebayo2018sanity}. Quantitative studies \citep{samek2016evaluating,
shrikumar2017learning} also show that these importance estimates are meaningful in predicting
model response to larger structured perturbations. These results suggest that input-gradients do
indeed capture relevant information regarding the underlying model. However in this work, we show that
input-gradients can be arbitrarily manipulated using the shift-invariance of softmax without
changing the underlying discriminative model, which calls into question the reliability of 
input-gradient based attribution methods for interpreting arbitrary black-box models.

Given that input-gradients can be arbitrarily structured, the reason for their highly structured
and explanatory nature in standard pre-trained models is puzzling. Why are input-gradients
relatively well-behaved when they can just as easily be arbitrarily structured, without affecting
discriminative model performance? What factors influence input-gradient structure in standard
deep neural networks?

To answer these, we consider the connections made between softmax-based discriminative
classifiers and generative models \citep{bridle1990probabilistic,grathwohl2020Your}, made by
viewing the logits of standard classifiers as un-normalized log-densities. This connection
reveals an alternate interpretation of input-gradients, as representing the log-gradients of a
class-conditional density model which is implicit within standard softmax-based deep models,
which we shall call the \textit{implicit density model}. This connection compels us to consider
the following hypothesis: perhaps input-gradients are highly structured because this implicit
density model is aligned with the `ground truth' class-conditional data distribution? The core of
this paper is dedicated to testing the validity of this hypothesis, whether or not
input-gradients do become more structured and explanatory if this alignment increases and vice
versa.

For the purpose of validating this hypothesis, we require mechanisms to increase or decrease the
alignment between the implicit density model and the data distribution. To this end, we consider
a generative modelling approach called score-matching, which reduces the density
modelling problem to that of local geometric regularization. Hence by using score-matching, we
are able to view commonly used geometric regularizers in deep learning as density modelling
methods. In practice, the score-matching objective is known for being computationally expensive and
unstable to train \citep{song2019generative, kingma2010regularized}. To this end, we also introduce 
approximations and regularizers which allow us to use score-matching on practical large-scale 
discriminative models. 

This work is broadly connected to the literature around unreliability of saliency methods. While most
such works consider how the explanations for nearly identical images can be arbitrarily different
\citep{domrowski2019explanations, Subramanya_2019_ICCV, zhang2020interpretable,
ghorbani2019interpretation}, our work considers how one may change the model itself to yield
arbitrary explanations without affecting discriminative performance. This is similar to 
\cite{heo2019fooling} who show this experimentally, whereas we provide an
analytical reason for why this happens relating to the shift-invariance of softmax.

The rest of the paper is organized as follows. We show in \S~\ref{sec:fooling} that it is trivial
to manipulate input-gradients of standard classifiers using the shift-invariance of softmax
without affecting the discriminative model. In \S~\ref{sec:score-matching} we state our main
hypothesis and describe the details of score-matching, present a tractable approximation for the
same that eliminates the need for expensive Hessian computations. \S~\ref{sec:connections}
revisits other interpretability tools from a density modelling perspective. Finally,
\S~\ref{sec:experiments} presents experimental evidence for the validity of the hypothesis that
improved alignment between the implicit density model and the data distribution can improve the
structure and explanatory nature of input-gradients.

\section{Input-Gradients are not Unique} 
\label{sec:fooling}

In this section, we show that it is trivial to manipulate input-gradients of discriminative deep
networks, using the well-known shift-invariance property of softmax. Here we shall make a
distinction between two types of input-gradients: \textit{logit-gradients} and
\textit{loss-gradients}. While logit-gradients are gradients of the pre-softmax output of a given
class w.r.t. the input, loss-gradients are the gradients of the loss w.r.t. the input. In both
cases, we only consider outputs of a single class, usually the target class.

Let $\X \in \R^D$ be a data point, which is the input for a neural network model $f: \R^D
\rightarrow \R^C$ intended for classification, which produces pre-softmax logits for $C$ classes.
The cross-entropy loss function for some class $1 \leq i \leq C, ~ i \in \mathbb{N}$
corresponding to an input $\X$ is given by $\ell(f(\X), i) \in \R_+$, which is shortened to
$\ell_{i}(\X)$ for convenience. Note that here the loss function subsumes the softmax function as
well. The logit-gradients are given by $\nabla_{\X} f_i(\X) \in \R^D$ for class $i$, while
loss-gradients are $\nabla_{\X} \ell_i(\X) \in \R^D$. Let the softmax function be $p(y=i|\X) =
{\exp(f_i(\X))} / {\sum_{j=1}^C \exp(f_j(\X))} $, which we denote as $p_i$ for
simplicity. Here, we make the observation that upon adding the same scalar function $g$ to all logits, the logit-gradients can arbitrarily change but the loss values do not.

\begin{observation}
Assume an arbitrary function $g: \R^D \rightarrow \R$. Consider another neural network function given by 
$\tilde{f}_i(\cdot) = f_i(\cdot) + g(\cdot), ~\text{for}~ 0 \leq i \leq C$, for which we obtain $\grad \tilde{f}_i(\cdot) = \grad f_i(\cdot) + \grad g(\cdot)$. For this, the corresponding loss values and loss-gradients are unchanged, i.e.; $\tilde{\ell}_i(\cdot) = \ell_i(\cdot)$ and $\grad \tilde{\ell}_i(\cdot) = \grad \ell_i(\cdot)$ as a consequence of the shift-invariance of softmax.
\end{observation}

This explains how the structure of logit-gradients can be arbitrarily changed: one simply needs to
add an arbitrary function $g$ to all logits. This implies that individual logit-gradients $\grad
f_i(\X)$ and logits $f_i(\X)$ are meaningless on their own, and their structure may be uninformative
regarding the underlying discriminative model. Despite this, a large fraction of work in
interpretable deep learning \citep{simonyan2013deep, selvaraju2017grad, smilkov2017smoothgrad,
fong2019understanding,srinivas2019full} uses individual logits and logit-gradients for saliency map
computation. We also provide a similar illustration in the supplementary material for the case of
loss-gradients, where we show that it is possible for loss-gradients to diverge significantly
even when the loss values themselves do not. 

These simple observations leave an open question: why are input-gradients highly structured and
explanatory when they can just as easily be arbitrarily structured, without affecting
discriminative model performance? Further, if input-gradients do not depend strongly on the
underlying discriminative function, what aspect of the model do they depend on instead? In the
section that follows, we shall consider a generative modelling view of discriminative
neural networks that offers insight into the information encoded by logit-gradients.

\section{Implicit Density Models Within Discriminative Classifiers}
\label{sec:score-matching}

Let us consider the following link between generative models and the softmax function. We first
define the following joint density on the logits $f_i$ of classifiers: $p_{\theta}(\X, y=i) =
\frac{\exp(f_i(\X; \theta))}{Z(\theta)} $, where $Z(\theta)$ is the partition function. We shall
henceforth suppress the dependence of $f$ on $\theta$ for brevity. Upon using Bayes' rule to
obtain $p_{\theta}(y=i \mid \X)$, we observe that we recover the standard softmax function. Thus
the logits of discriminative classifiers can alternately be viewed as un-normalized log-densities
of the joint distribution. Assuming equiprobable classes, we have $p_{\theta}(\X \mid y = i) =
\frac{\exp(f_i(\X))}{Z(\theta)/C}$, which is the quantity of interest for us. Thus while the
logits represent un-normalized log-densities, logit-gradients represent the score function, i.e.;
$\nabla_x \log p_{\theta} (\X \mid y=i) = \nabla_x f_i(\X)$, which avoids dependence on the
partition function $Z(\theta)$ as it is independent of $\X$.

This viewpoint naturally leads to the following hypothesis, that perhaps the reason for the
highly structured and explanatory nature of input-gradients is that the implicit density model
$p_{\theta}(\X \mid y)$ is close to that of the ground truth class-conditional data distribution
$p_\text{data}(\X \mid y)$? We propose to test this hypothesis explicitly using score-matching 
as a density modelling tool.

\begin{hypothesis}
   \label{hyp:hypothesis}
(Informal) Improved alignment of the implicit density model to the ground
truth class-conditional density model improves input-gradient interpretability via both qualitative
and quantitative measures, whereas deteriorating this alignment has the opposite effect.
\end{hypothesis}

\subsection{Score-Matching}

Score-matching \citep{hyvarinen2005estimation} is a generative modelling objective that focusses
solely on the derivatives of the log density instead of the density itself, and thus does not
require access to the partition function $Z(\theta)$. Specifically, for our case we have $\grad
\log p_{\theta}(\X \mid y= i) = \grad f_i(\X)$, which are the logit-gradients.

Given i.i.d. samples $\mathcal{X} = \{x_i \in \R^D \}$ from a latent data distribution
$p_{data}(\X)$, the objective of generative modelling is to recover this latent distribution
using only samples $\mathcal{X}$. This is often done by training a parameterized distribution
$p_{\theta}(\X)$ to align with the latent data distribution $p_{data}(\X)$. The score-matching objective 
instead aligns the gradients of log densities, as given below.

\begin{eqnarray}
    J(\theta) &=& \E_{p_{data}(\X)} \frac{1}{2}\| \grad \log p_{\theta}(\X) - \grad \log p_{data}(\X) \|^2_2 \label{eqn:score-matching}\\
    &=& \E_{p_{data}(\X)} \left( \text{trace}(\grad^2 \log p_{\theta}(\X))  + \frac{1}{2} \| \grad \log p_{\theta}(\X) \|^2_2 \right) + \mathtt{const} \label{eqn:score-matching-hessian}
\end{eqnarray}

The above relationship is proved \citep{hyvarinen2005estimation} using integration by parts. This
is a consistent objective, \emph{i.e,} $J(\theta) = 0 \iff p_{data} = p_{\theta}$. This approch
is appealing also because this reduces the problem of generative modelling to that of
regularization of the local geometry of functions, i.e.; the resulting terms only depend on the
point-wise gradients and Hessian-trace.

\subsection{Efficient estimation of Hessian-trace}

In general, equation \ref{eqn:score-matching-hessian} is
intractable for high-dimensional data due to the Hessian trace term. To address this, we
can use the Hutchinson's trace estimator \citep{hutchinson1990stochastic} to efficiently compute
an estimate of the trace by using random projections, which is given by: $\text{trace}(\grad^2
\log p_{\theta}(\X)) = \E_{\f{v} \sim \mathcal{N}(0,\mathtt{I})} ~\f{v}^\mathtt{T} ~\grad^2 \log
p_{\theta}(\X) ~ \f{v} $. This estimator has been previously applied to score-matching
\citep{song2019sliced}, and can be computed efficiently using Pearlmutter's trick
\citep{pearlmutter1994fast}. However, this trick still requires \textbf{two backward passes} for a
single monte-carlo sample, which is computationally expensive. To further improve computational
efficiency, we introduce the following approximation to Hutchinson's estimator using a Taylor
series expansion, which applies to small values of $\sigma \in \R$.

  \newcommand{\Ev}{\E_{\f{v} \sim \mathcal{N}(0,\sigma^2 \mathtt{I})}}

 \begin{eqnarray}
    \E_{\f{v} \sim \mathcal{N}(0,\mathtt{I})} ~\f{v}^\mathtt{T} \grad^2 \log p_{\theta}(\X) \f{v}  
    %\frac{1}{\sigma^2}  ~ \Ev \f{v}^\mathtt{T} \grad^2 \log p_{\theta}(\X) \f{v} \nonumber
    %\\
    & \approx & \frac{2}{\sigma^2} ~\Ev \left( \log p_{\theta}(\X + \f{v}) - \log p_{\theta}(\X) - \nabla_{x} \log p_{\theta}(\X)^\mathtt{T} \f{v} \right) \nonumber \\
    & = & \frac{2}{\sigma^2} ~\Ev \left( \log p_{\theta} (\X + \f{v}) - \log p_{\theta} (\X) \right) \label{eqn:taylor-trick}
 \end{eqnarray}

Note that equation \ref{eqn:taylor-trick} involves a difference of log probabilities, which is
independent of the partition function. For our case, $\log p_{\theta} (\X + \f{v} | y=i) - \log
p_{\theta} (\X | y=i) = f_i(\X + \f{v}) - f_i(\X)$. We have thus considerably simplified and
speeded-up the computation of the Hessian trace term, which now can be approximated with
\textbf{no backward passes}, but using only a single additional forward pass. We present 
details regarding the variance of this estimator in the supplementary material. A concurrent approach \citep{pang2020efficient} also presents a similar algorithm, however it is applied 
primarily to Noise Contrastive Score Networks \citep{song2019generative} and Denoising Score Matching 
\citep{vincent2011connection}, whereas we apply it to vanilla score-matching on discriminative models.

\subsection{Stabilized Score-matching} 
\label{sec:stability}
In practice, a naive application of score-matching objective is unstable, causing the
Hessian-trace to collapse to negative infinity. This occurs because the finite-sample variant of
equation \ref{eqn:score-matching} causes the model to `overfit' to a mixture-of-diracs
density, which places a dirac-delta distribution at every data point. Gradients of such a distribution
are undefined, causing training to collapse. To overcome this, regularized
score-matching \citep{kingma2010regularized} and noise conditional score networks
\citep{song2019generative} propose to add noise to inputs for score-matching to make the problem 
well-defined. However, this did not help for our case. Instead, we use a heuristic where 
we add a small penalty term proportional to the square of the Hessian-trace. This discourages the 
Hessian-trace becoming too large, and thus stabilizes training. 

%Similar to these works, we also add noise with standard deviation $\tau$ to the inputs
%during score-matching, and find that it helps stabilize training.

\section{Implications of the Density Modelling Viewpoint}
\label{sec:connections}
In the previous section we related input-gradients to the implicit density model, thus linking
gradient interpretability to density modelling through our hypothesis. In this section, we
consider two other interpretability tools: activity maximization and the pixel perturbation test, and
show how these can interpreted from a density modelling perspective. These perspectives also
enable us to draw parallels between score-matching and adversarial training.

\subsection{Activity Maximization as Sampling from the Implicit Density Model}
The canonical method to obtain samples from score-based generative models is via Langevin
sampling \citep{welling2011bayesian, song2019generative}, which involves performing
gradient ascent on the density model with noise added to the gradients. Without this added noise,
the algorithm recovers the modes of the density model.

We observe that activity maximization algorithms used for neural network visualizations are
remarkably similar to this scheme. For instance, \cite{simonyan2013deep} recover inputs which
maximize the logits of neural networks, thus exactly recovering the modes of the implicit density
model. Similarly, deep-dream-like methods \citep{mahendran2016visualizing,
nguyen2016synthesizing, mordvintsev2015inceptionism} extend this by using ``image priors'' to
ensure that the resulting samples are closer to the distribution of natural images, and by adding
structured noise to the gradients in the form of jitter, to obtain more visually pleasing
samples. From the density modelling perspective, we can alternately view these visualization
techniques as biased sampling methods for score-based density models trained on natural images.
However, given the fact that they draw samples from the implicit density model, their utility in
interpreting discriminative models may be limited.

\subsection{Pixel Perturbation as a Density Ratio Test}
\label{subsec:pixelpert}
A popular test for saliency map evaluation is based on pixel perturbation
\citep{samek2016evaluating}. This involves first selecting the least-relevant (or most-relevant)
pixels according to a saliency map representation, `deleting' those pixels and measuring
the resulting change in output value. Here, deleting a pixel usually involves replacing the pixel
with a non-informative value such as a random or a fixed constant value. A good saliency method
identifies those pixels as less relevant whose deletion does not cause a large change in output
value.

We observe that this change in outputs criterion is identical to the density ratio,
\latin{i.e.}, $\log \left( p_{\theta} (\X + \f{v} | y=i) / p_{\theta} (\X | y=i) \right) = f_i(\X
+ \f{v}) - f_i(\X)$. Thus when logits are used for evaluating the change in outputs
\citep{samek2016evaluating, ancona2018towards}, the pixel perturbation test exactly measures the
density ratio between the perturbed image and the original image. Thus if a perturbed image has a
similar density to that of the original image under the implicit density model, then the saliency
method that generated these perturbations is considered to be explanatory. Similarly,
\cite{fong2019understanding} optimize over this criterion to identify pixels whose removal causes
minimal change in logit activity, thus obtaining perturbed images with a high implicit density
value similar to that of activity maximization. Overall, this test captures sensitivity of the
implicit density model, and not the underlying discriminative model which we wish to interpret.
We thus recommend that the pixel perturbation test always be used in conjunction with either the
change in output probabilities, or the change in the accuracy of classification, rather than 
the change in logits.

\subsection{Connecting Score-Matching to Adversarial Training}

Recent works in adversarial machine learning \citep{etmann2019connection, engstrom2019adversarial,
santurkar2019image, kaur2019perceptually, ross2017improving} have observed that saliency map
structure and samples from activation maximization are more perceptually aligned for
adversarially trained models than for standard models. However it is unclear from these works why
this occurs. Separate from this line of work, \cite{chalasani2018concise} also connect regularization of a 
variant of integrated gradients with adversarial training, suggesting a close interplay between the two. 

We notice that these properties are shared with score-matched models, or models
trained such that the implicit density model is aligned with the ground truth. Further, we note
that both score-matching and adversarial training are often based on local geometric
regularization, usually involving regularization of the gradient-norm \citep{ross2017improving,
jakubovitz2018improving}, and training both the discriminative model and the implicit density
model \citep{grathwohl2020Your} has been shown to improve adversarial robustness. From these
results, we can conjecture that training the implicit density model via score-matching may have
similar outcomes as adversarial training. We leave the verification and proof of this conjecture 
to future work.

\section{Experiments}
\label{sec:experiments}

In this section, we present experimental results to show the efficacy of score-matching and the
validation of the hypothesis that density alignment influences the gradient explanation quality. For
experiments, we shall consider the CIFAR100 dataset. We present experiments with CIFAR10 in the
supplementary section. Unless stated otherwise, the network structure we use shall be a 18-layer
ResNet that achieves 78.01\% accuracy on CIFAR100, and the optimizer used shall be SGD with
momentum. All models use the softplus non-linearity with $\beta = 10$, which is necessary to
ensure that the Hessian is non-zero for score-matching. Before proceeding with our experiments,
we shall briefly introduce the score-matching variants we shall be using for comparisons.
 
\paragraph{Score-Matching}
We propose to use the score-matching objective as a regularizer in neural network training to
\textbf{increase} the alignment of the implicit density model to the ground truth, as shown in
equation \ref{eqn:final_loss}, with the stability regularizer discussed in \S
\ref{sec:stability}. For this, we use a regularization constant $\lambda = 1e-3$. This model achieves 
$72.20\%$ accuracy on the test set, which is a drop of about $5.8\%$ compared to the original model. 
In the supplementary material, we perform a thorough hyper-parameter sweep and show that it is possible 
to obtain better performing models.

\begin{gather}
    h(\X) :=\frac{2}{\sigma^2} \Ev \left( f_i(\X + \f{v}) - f_i(\X) \right)  \nonumber \\
    \underbrace{\ell_{reg}(f(\X), i)}_{\text{regularized loss}} = \underbrace{\ell(f(\X), i)}_{\text{cross-entropy}} + \lambda \left( \underbrace{ \overbrace{h(\X)}^{\text{Hessian-trace}}  + \frac{1}{2} \overbrace{\| \grad f_i(\X) \|^2_2}^{\text{gradient-norm}}}_{\text{score-matching}} + \underbrace{\overbrace{\mu}^{10^{-4}} h^2(\X)}_{\text{stability regularizer}} \right) \label{eqn:final_loss}
\end{gather}

\paragraph{Anti-score-matching}
We would like to have a tool that can \textbf{decrease} the alignment between the implicit
density model and the ground truth. To enable this, we propose to maximize the hessian-trace, in
an objective we call \textit{anti-score-matching}. For this, we shall use a the clamping function
on hessian-trace, which ensures that its maximization stops after a threshold is reached. We use
a threshold of $\tau = 1000$, and regularization constant $\lambda = 1e-4$. This model achieves
an accuracy of $74.87\%$.

\paragraph{Gradient-Norm regularization} 
We propose to use gradient-norm regularized models as
another baseline for comparison, using a regularization constant of $\lambda = 1e-3$. This model
achieves an accuracy of $76.60\%$.

\subsection{Evaluating the Efficacy of Score-Matching and Anti-Score-Matching}
Here we demonstrate that training with score-matching / anti-score-matching is possible,
and that such training improves / deteriorates the quality of the implicit density models respectively
as expected.

\subsubsection{Density Ratios}

One way to characterize the generative behaviour of models is to compute likelihoods on data points.
However this is intractable for high-dimensional problems, especially for un-normalized models.
We observe although that the densities $p(\X \mid y=i)$ themselves are intractable, we can easily
compute density ratios $p(\X + \eta \mid y=i) / p(\X \mid y=i) = \exp(f_i(\X + \eta) - f_i(\X))$
for a random noise variable $\eta$. Thus, we propose to plot the graph of density ratios locally
along random directions. These can be thought of as local cross-sections of the density sliced at
random directions. We plot these values for gaussian noise $\eta$ for different standard
deviations, which are averaged across points in the entire dataset.

\begin{wrapfigure}{L}{0.45\textwidth}
    \centering
    \includegraphics[width=0.9\linewidth, trim={0.4cm 0cm 0cm 0cm}, clip]{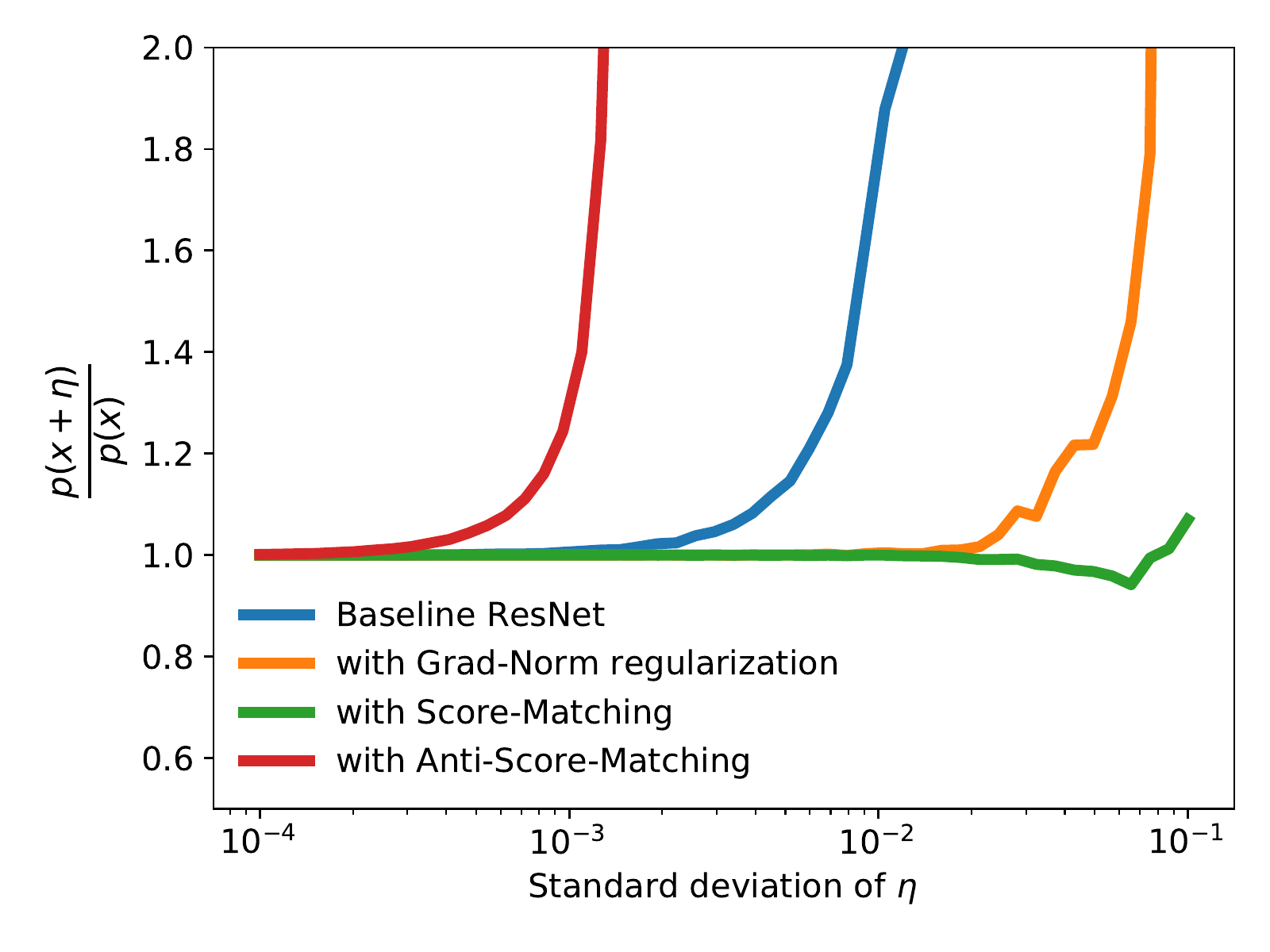}
    \caption{Plots of density ratios representing local density profiles across varying levels of
    noise added to the input (lower is better). We observe that score-matched model is robust to
    a larger range of noise values, while anti-score-matching is very sensitive even to small amounts
    of noise.}
    \label{fig:density-reg}
\end{wrapfigure}

In Figure \ref{fig:density-reg}, we plot the density ratios upon training on the CIFAR100
dataset. We observe that the baseline model assigns \textit{higher} density values to noisy
inputs than real inputs. With anti-score-matching, we observe that the density profile grows
still steeper, assigning higher densities to inputs with smaller noise. Gradient-norm regularized
models and score-matched models improve on this behaviour, and are robust to larger amounts of
noise added. Thus we are able to obtain penalty terms that can both improve and deteriorate the
density modelling behaviour within discriminative models.

\subsubsection{Sample Quality}
\label{subsec:sample-quality}

We are interested in recovering modes of our density models while having access to only the
gradients of the log density. For this purpose, we apply gradient ascent on the log probability
$\log p(\X \mid y=i) = f_i(\X)$, similar to activity maximization. Our results are shown in
Figure \ref{fig:samples}. We observe that samples from the score-matched and gradient-norm
regularized models are significantly less noisy than other models.

We also propose to qualitatively measure the sample quality using the GAN-test approach
\citep{shmelkov2018good}. This test proposes to measure the discriminative accuracy of generated
samples via an independently trained discriminative model. In contrast with more popular metrics
such as the inception-score, this captures sample quality rather than diversity, which is what we
are interested in. We show the results in table \ref{tab:sample_accuracy}, which confirms the
qualitative trend seen in samples above. Surprisingly, we find that gradient-norm regularized models 
perform better than score-matched models. This implies that such models are
able to implicitly perform density modelling without being explicitly trained to do so.
We leave further investigation of this phenomenon to future work. %We also independently observe
%that samples from softplus trained models tend to perform significantly better than those
%obtained from ReLU models. We believe that the improved functional smoothness in softplus models
%contributes to this effect.

\begin{table}%{L}{0.5\textwidth}
\begin{center}
  \begin{tabular}{ c|c } 
   \textbf{Model} & \textbf{GAN-test (\%)} \\ 
   \hline
   Baseline ResNet & 59.47  \\
   + Anti-Score-Matching & 16.40  \\
   + Gradient Norm-regularization & \textbf{80.07}  \\
   + Score-Matching & 72.75 
  \end{tabular}
  \end{center}
  \caption{GAN-test scores (higher is better) of class-conditional samples generated from various ResNet-18 models (see \S~\ref{subsec:sample-quality}). We observe that samples from gradient-norm regularized models and score-matched 
  models achieve much better accuracies than the baselines and anti-score-matched models.}
  \label{tab:sample_accuracy}
\end{table}

\begin{figure}
    \centering
    \captionsetup[subfigure]{justification=centering}
    \begin{subfigure}[t]{0.2\textwidth}
        \centering
        \includegraphics[width=0.9\linewidth]{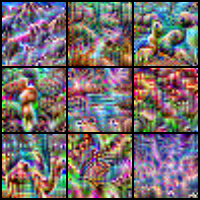}
        \caption{Baseline ResNet}
    \end{subfigure}%
    \begin{subfigure}[t]{0.2\textwidth}
        \centering
        \includegraphics[width=0.9\linewidth]{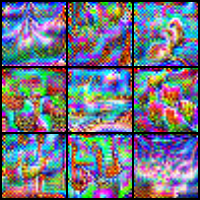}
        \caption{With anti score-matching}
    \end{subfigure}%
    \begin{subfigure}[t]{0.2\textwidth}
        \centering
        \includegraphics[width=0.9\linewidth]{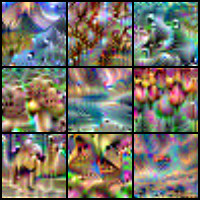}
        \caption{With Gradient-norm regularization}
    \end{subfigure}%
    \begin{subfigure}[t]{0.2\textwidth}
        \centering
        \includegraphics[width=0.9\linewidth]{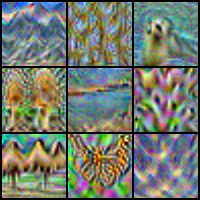}
        \caption{With score-matching}
    \end{subfigure}
\caption{Samples generated from various models by performing gradient ascent on random inputs (see \S~-\ref{subsec:sample-quality}).
While none of the generated samples are realistic, samples obtained from score-matched and
gradient-norm regularized models are smoother and less noisy.}

\label{fig:samples}
\end{figure}

\subsection{Evaluating the Effect of Density Alignment on Gradient Explanations}

Here we shall evaluate the gradient explanations of various models. First,
we shall look at quantitative results on a discriminative variant of the pixel perturbation test.
Second, we visualize the gradient maps to assess qualitative differences between them.

\subsubsection{Quantitative Results on Discriminative Pixel Perturbation}
\label{subsec:pixpert-results}
As noted in \ref{subsec:pixelpert}, it is recommended to use the pixel perturbation test using
accuracy changes, and we call this variant as \textit{discriminative pixel perturbation}. We
select the least relevant pixels and replace them with the mean pixel value of the image, note
down the accuracy of the model on the resulting samples. We note that this test is only used so
far to compare different saliency methods for the same underlying model. However, we here seek to
compare saliency methods across models. For this we consider two experiments. First, we perform
the pixel perturbation experiment with each of the four trained models on their own
input-gradients and plot the results in Figure \ref{fig:pixpert-seperate}. These results indicate that
the input-gradients of score-matched and gradient-norm regularized models are better equipped to
identify least relevant pixels in
this model. However, it is difficult to completely disentangle the robustness benefits of such
score-matched models against improved identification of less relevant pixels through such a plot.

To this end, we conduct a second experiment in Figure \ref{fig:pixpert-proxy}, where we use
input-gradients obtained from these four trained models to explain the same standard baseline
ResNet model. This disentangles the robustness of different models as inputs to the same model is
perturbed in all cases. Here also we find that gradients from score-matched and gradient-norm
regularized models explain behavior of standard baseline models better than the gradients of the
baseline model itself. Together, these tests show that training with score-matching indeed
produces input-gradients that quantitatively more
explanatory than baseline models.

\begin{figure}
   \captionsetup[subfigure]{justification=centering}
   \centering
   \begin{subfigure}{0.4\textwidth}
      \includegraphics[width=0.98\textwidth]{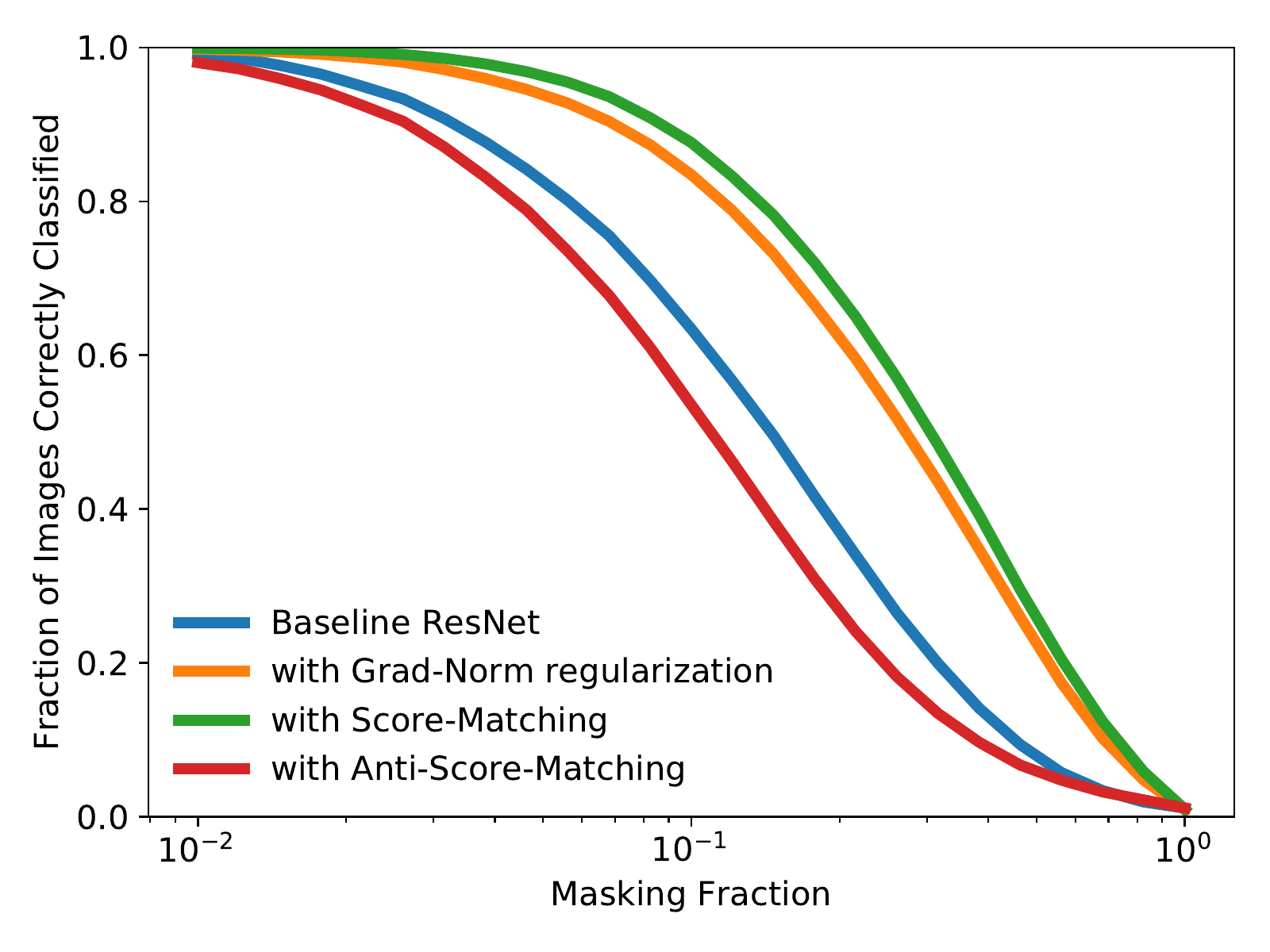}
      \caption{Models evaluated with their own gradients}
      \label{fig:pixpert-seperate}
   \end{subfigure}%
   \begin{subfigure}{0.4\textwidth}
      \includegraphics[width=0.98\textwidth]{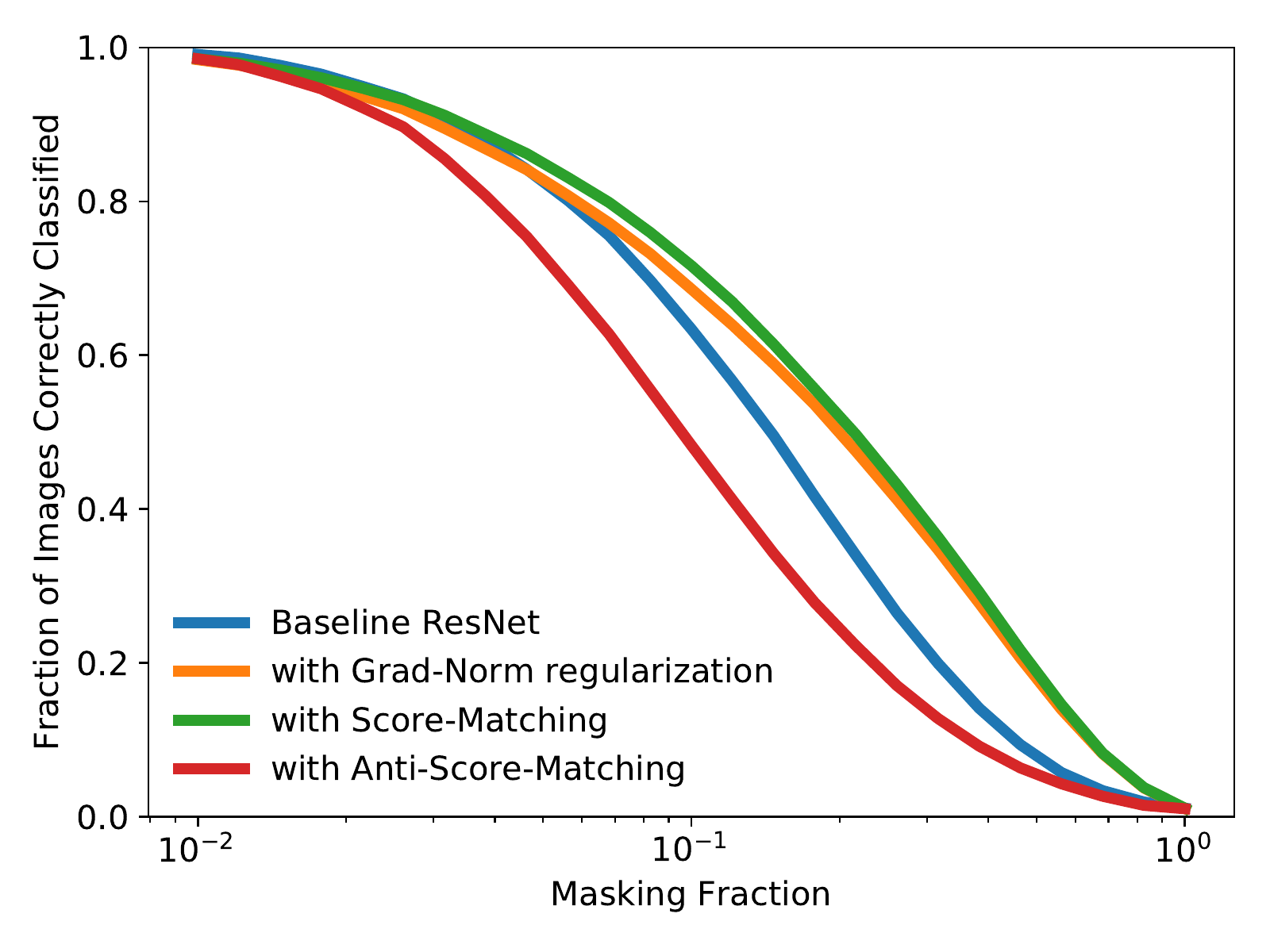}
      \caption{Baseline ResNet evaluated with gradients of \textit{different} models}
      \label{fig:pixpert-proxy}
   \end{subfigure}
   \caption{Discriminative pixel perturbation results (higher is better) on the CIFAR100 dataset (see \S~\ref{subsec:pixpert-results}).
   We see that score-matched and gradient-norm regularized models best explain model behaviour in both 
   cases, while the anti-score-matched model performs the worst. This agrees with the hypothesis (stated in \S~\ref{sec:score-matching}) that 
   alignment of implicit density models improves gradient explanations and vice versa.}

\end{figure}

\subsubsection{Qualitative Gradient Visualizations}

We visualize the structure of logit-gradients of different models in Figure \ref{fig:gradients}.
We observe that gradient-norm regularized model and score-matched model have highly perceptually
aligned gradients, when compared to the baseline and anti-score-matched gradients, corroborating
the quantitative results.

\begin{figure}[h]
    \captionsetup[subfigure]{justification=centering}
    \centering
    \begin{subfigure}[t]{0.18\textwidth}
        \centering
        \includegraphics[width=0.9\linewidth]{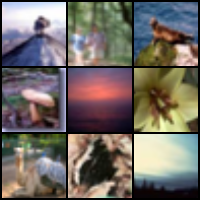}
        \caption{Input Image}
    \end{subfigure}%
    \begin{subfigure}[t]{0.18\textwidth}
        \centering
        \includegraphics[width=0.9\linewidth]{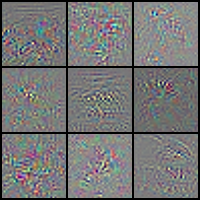}
        \caption{Baseline ResNet}
    \end{subfigure}%
    \begin{subfigure}[t]{0.18\textwidth}
        \centering
        \includegraphics[width=0.9\linewidth]{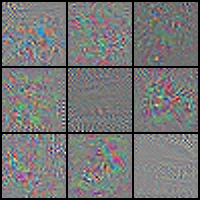}
        \caption{With Anti score-matching}
    \end{subfigure}%
    \begin{subfigure}[t]{0.18\textwidth}
        \centering
        \includegraphics[width=0.9\linewidth]{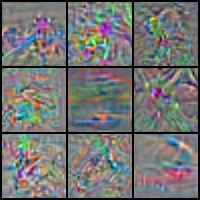}
        \caption{With Gradient-norm regularization}
    \end{subfigure}%
    \begin{subfigure}[t]{0.18\textwidth}
        \centering
        \includegraphics[width=0.9\linewidth]{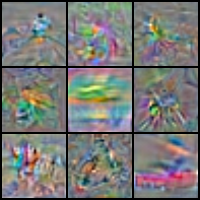}
        \caption{With Score-matching}
    \end{subfigure}
\caption{Visualization of input-gradients of different models. We observe that gradients of score-matched
and gradient-norm regularized models are more perceptually aligned than the others, with the gradients of the anti-score-matched model being the noisiest.
This qualitatively verifies the hypothesis stated in \S~\ref{sec:score-matching}.}
\label{fig:gradients}
\end{figure}

\section{Conclusion}

In this paper, we investigated the cause for the highly structured and explanatory nature of
input-gradients in standard pre-trained models, and showed that alignment of the implicit density
model with the ground truth data density is a possible cause. This density modelling
interpretation enabled us to view canonical approaches in interpretability such as gradient-based
saliency methods, activity maximization and the pixel perturbation test through a density
modelling perspective, showing that these capture information relating to the implicit density
model, not the underlying discriminative model which we wish to interpret. This calls for a need
to re-think the role of these tools in the interpretation of discriminative models. For practitioners, 
we believe it is best to avoid usage of logit gradient-based tools, for interpretability. If unavoidable, 
it is recommended to use only gradient-norm regularized or score-matched models, as input-gradients of these 
models produce more reliable estimates of the gradient of the underlying distribution. As our experiments 
show, these may be a useful tool even though they are not directly related to the discriminative model.

However, our work still does not answer the question of why pre-trained models may have their
implicit density models aligned with ground truth in the first place. One possible reason could be the 
the presence of an implicit gradient norm regularizer in standard SGD, similar to that shown 
independently by \cite{barrett2020implicit}. Another open question is to understand why
gradient-norm regularized models are able to perform implicit density modelling as observed in
our experiments in \S~\ref{subsec:sample-quality}, which lead to improved gradient explanations.

\bibliography{references}
\bibliographystyle{iclr2021_conference}

\clearpage

\appendix
\part*{Appendix}

\section{Fooling Gradients is simple}

\begin{observation}
    Assume an arbitrary function $g: \R^D \rightarrow \R$. Consider another neural network function given by 
    $\tilde{f}_i(\cdot) = f_i(\cdot) + g(\cdot), ~\text{for}~ 0 \leq i \leq C$, for which we obtain $\grad \tilde{f}_i(\cdot) = \grad f_i(\cdot) + \grad g(\cdot)$. For this, the corresponding loss values and loss-gradients are unchanged, i.e.; $\tilde{\ell}_i(\cdot) = \ell_i(\cdot)$ and $\grad \tilde{\ell}_i(\cdot) = \grad \ell_i(\cdot)$.
\end{observation}

\begin{proof}
The following expressions relate the loss and neural network function outputs, for the 
case of cross-entropy loss and usage of the softmax function.

    \begin{eqnarray}
    \ell_i(\X) &=& - f_i(\X) + \log \left( \sum_{j=1}^C \exp(f_j(\X)) \right) \label{eqn:loss}\\
    \grad \ell_i(\X) &=& - \grad f_i(\X) + \sum_{j=1}^C p_j \grad f_j(\X)
    \end{eqnarray}

Upon replacing $f_i$ with $\tilde{f}_i = f_i + g$, the proof follows.
\end{proof}

\subsection{Manipulating Loss-Gradients}
Here, we show how we can also change loss-gradients arbitrarily without significantly 
changing the loss values themselves. In this case, the trick is to add a high frequency low 
amplitude sine function to the loss. 

\begin{observation}
   Consider $g(\X) = \epsilon \sin(m \X)$, and $\tilde{\ell}_i(\X) = \ell_i(\X) + g(\X)$, for
   $\epsilon, m \geq \R_+$ and $\X \in \R^D$. Then, it is easy to see that
   $ | \tilde{\ell}_i(\X) - \ell_i(\X) | \leq \epsilon$, and $\| \grad \tilde{\ell}_i(\X) - \grad \ell_i(\X) \|_1 \leq m  \times \epsilon \times D $.
\end{observation}

Thus two models with losses differing by some small $\epsilon$ can have gradients differing by $m
\times \epsilon \times D$. For $m \rightarrow \infty$ and a fixed $\epsilon$, the gradients can
diverge significantly. Thus, loss-gradients are also unreliable, as two models with very similar
loss landscapes and hence discriminative abilities, can have drastically different
loss-gradients.

This simple illustration highlights the fact that gradients of high-dimensional black-box models
are not well-behaved in general, and this depends on both the model smoothness and the
high-dimensionality of the inputs. Further, loss values and loss-gradients for highly confident
samples are close to zero. Thus any external noise added (due to stochastic training, for
instance) can easily dominate the loss-gradient terms even when smoothness conditions (small $m$)
are enforced.

\section{Score-Matching Approximation}

We consider the approximation derived for the estimator of the Hessian trace, which is first
derived from Hutchinson's trace estimator \cite{hutchinson1990stochastic}. We replace $\log p_{\theta}(\X)$ terms used in the
main text with $f(\X)$ terms here for clarity. The Taylor series trick for approximating the
Hessian-trace is given below.

\begin{eqnarray}
   \E_{\f{v} \sim \mathcal{N}(0,\mathtt{I})} ~\f{v}^\mathtt{T} \grad^2 f(\X) \f{v} &=& 
   \frac{1}{\sigma^2}  ~ \Ev \f{v}^\mathtt{T} \grad^2 f(\X) \f{v} \nonumber
   \\
   & = & \frac{2}{\sigma^2} ~\Ev \left( f(\X + \f{v}) - f(\X) - \nabla_{x} f(\X)^\mathtt{T} \f{v}  + \mathcal{O}(\sigma^3) \right) \nonumber \\
\end{eqnarray}

As expected, the approximation error vanishes in the limit of small $\sigma$. Let us now consider the finite sample variants of this estimator, with $N$ samples. We shall call this the \textit{Taylor Trace Estimator}. 

\newcommand{\Sv}{\sum_{i=1}^N}

\begin{equation}
    \text{Taylor Trace Estimator (TTE)} = \frac{2}{N \sigma^2} ~\Sv \bigl( f (\X + \f{v}_i) - f(\X) \bigl) ~~~~ \text{s.t.} ~~~\f{v}_i \sim \mathcal{N}(0,\sigma^2 \mathtt{I})
\end{equation}

We shall henceforth suppress the dependence on $i$ for brevity. For this estimator, we can
compute its variance for quadratic functions $f$, where higher-order Taylor expansion terms are
zero. We make the following observation.

\begin{observation}
    For quadratic functions $f$, the variance of the Taylor Trace Estimator is greater than the variance of the Hutchinson estimator by an amount at most equal to $4 \sigma^{-2} \| \grad f(\X) \|^2$.
\end{observation}

\begin{proof}

\begin{eqnarray*}
    \text{Var(T.T.E.)} &=& \frac{1}{\sigma^4} ~\E_v \left(\frac{2}{N} \Sv \bigl( f (\X + \f{v}) - f(\X) \bigl)  - \E_v \f{v}^\mathtt{T} \grad^2 f(\X) \f{v} \right)^2 \\
    & = & \frac{1}{\sigma^4} ~\E_v \biggl(\frac{2}{N}  \Sv \bigl( f (\X + \f{v}) - f(\X) \bigl)  - \frac{1}{N}  \Sv \f{v}^\mathtt{T} \grad^2 f(\X) \f{v} \\ 
    && + \frac{1}{N} \Sv \f{v}^\mathtt{T} \grad^2 f(\X) \f{v} - \E_v \f{v}^\mathtt{T} \grad^2 f(\X) \f{v} \biggl)^2 \\
    & = & \frac{1}{\sigma^4} ~\E_v \biggl(\frac{2}{N}  \Sv \bigl( f (\X + \f{v}) - f(\X) \bigl)  - \frac{1}{N} \Sv \f{v}^\mathtt{T} \grad^2 f(\X) \f{v} \biggl)^2 \\
    && + \frac{1}{\sigma^4} \E_v \biggl(\frac{1}{N}  \Sv \f{v}^\mathtt{T} \grad^2 f(\X) \f{v} - \E_v \f{v}^\mathtt{T} \grad^2 f(\X) \f{v} \biggl)^2 
\end{eqnarray*}

Thus we have decomposed the variance of the overall estimator into two terms: the first captures the 
variance of the Taylor approximation, and the second captures the variance of the Hutchinson estimator. 

Considering only the first term, i.e.; the variance of the Taylor approximation, we have:

\begin{eqnarray*}
     \frac{1}{N\sigma^4} ~\E_v \biggl(2 \Sv \bigl( f (\X + \f{v}) - f(\X) \bigl)  - \Sv \f{v}^\mathtt{T} \grad^2 f(\X) \f{v} \biggl)^2 & = & \frac{4}{N\sigma^4} ~ \E_v \biggl( \Sv \grad f(\X)^T \f{v} \biggl)^2   \\
     & \leq & \frac{4}{\sigma^4} \| \grad f(\X) \|^2 \E_v \| \f{v} \|^2 \\
     & = & 4 \sigma^{-2} \| \grad f(\X) \|^2
\end{eqnarray*}

The intermediate steps involve expanding the summation, noticing that pairwise terms cancel, and applying
the Cauchy-Schwartz inequality.
\end{proof}

Thus we have a trade-off: a large $\sigma$ results in lower estimator variance but a large Taylor
approximation error, whereas the opposite is true for small $\sigma$. However for functions with small
gradient norm, both the estimator variance and Taylor approximation error is small for small $\sigma$.
We note that when applied to score-matching \cite{hyvarinen2005estimation}, the gradient norm of
the function is also minimized.  This implies that in practice, the gradient norm of the
function is likely to be low, thus resulting in a small estimator variance even for small $\sigma$. 
The variance of the Hutchinson estimator is given below
for reference \cite{hutchinson1990stochastic, avron2011randomized}:

\begin{equation*}
    \text{Var(Hutchinson)}  =  \frac{2}{N} \|\grad^2 f(\X) \|^2_F
\end{equation*}

\section{Evaluating Effect of Score-Matching on Gradient Explanations (on CIFAR10)}
We repeat the pixel perturbation experiments on the CIFAR10 dataset and we observe similar 
qualitative trends. In both cases, we observe that score-matched and gradient norm regularized 
models have more explanatory gradients, while anti-score-matched model contains the least 
explanatory gradients. We also present visualization results of input-gradients of various 
models for reference.

\begin{figure}[h]
   \captionsetup[subfigure]{justification=centering}
   \centering
   \begin{subfigure}{0.4\textwidth}
      \includegraphics[width=0.98\textwidth]{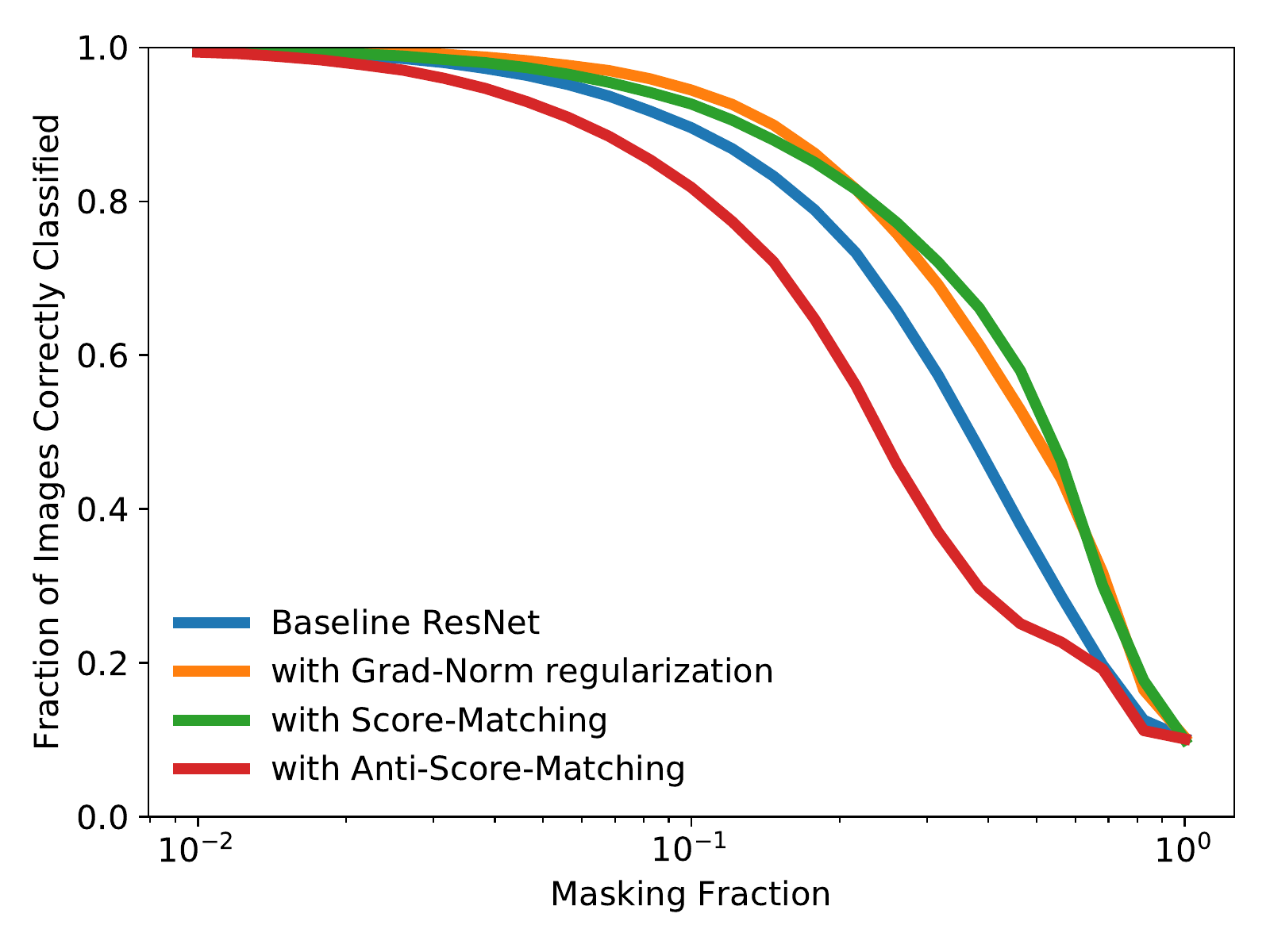}
      \caption{Models evaluated with their own gradients}
   \end{subfigure}%
   \begin{subfigure}{0.4\textwidth}
      \includegraphics[width=0.98\textwidth]{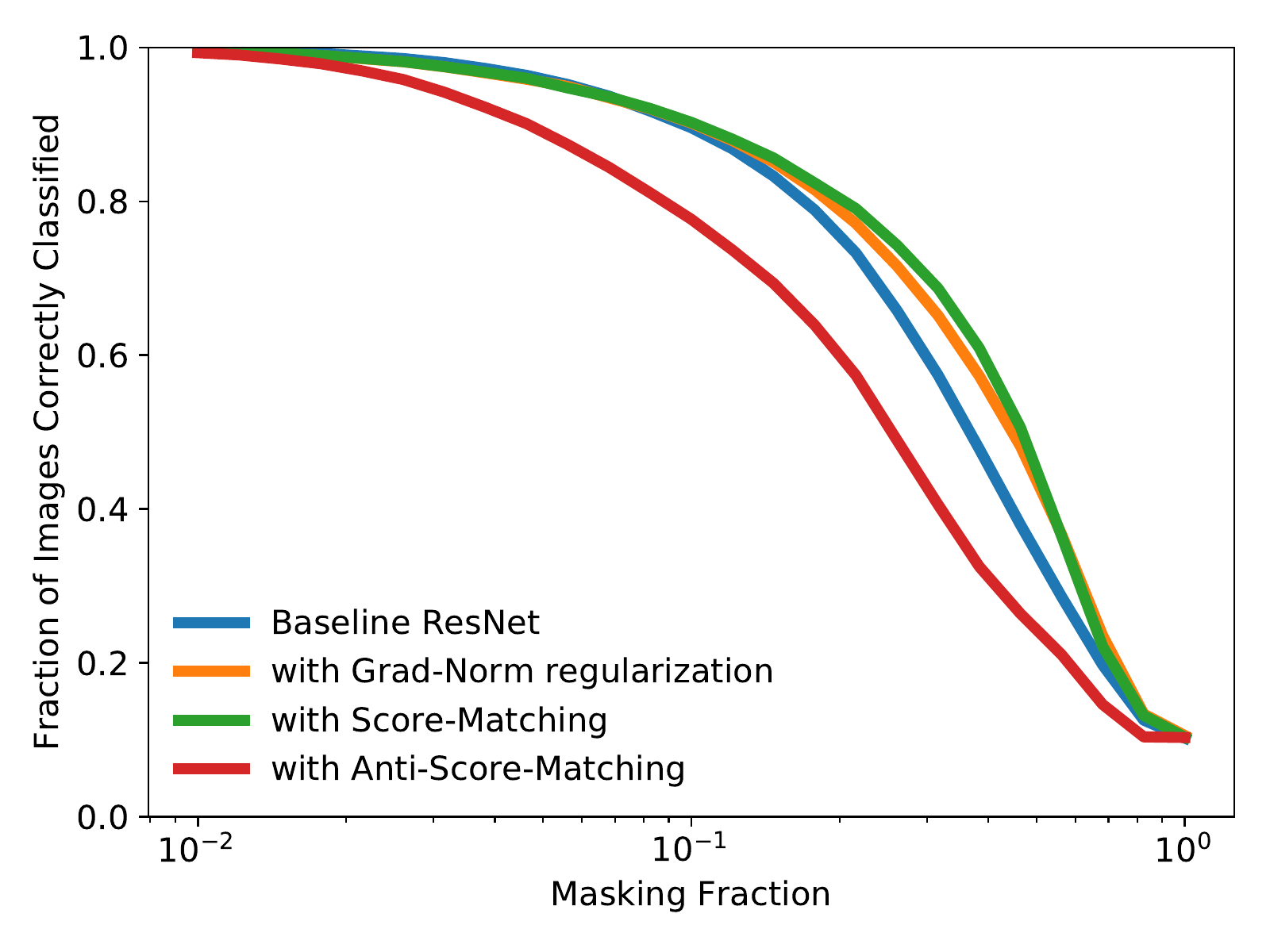}
      \caption{Baseline ResNet evaluated with gradients of \textit{different} models}
   \end{subfigure}
   \caption{Discriminative pixel perturbation results (higher is better) on the CIFAR10 dataset (see \S~\ref{subsec:pixpert-results}).
   We see that score-matched and gradient-norm regularized models best explain model behaviour in both 
   cases, while the anti-score-matched model performs the worst. This agrees with the hypothesis (stated in \S~\ref{sec:score-matching}) that 
   alignment of implicit density models improves gradient explanations and vice versa.}
\end{figure}

\begin{figure}[h]
    \captionsetup[subfigure]{justification=centering}
    \centering
    \begin{subfigure}[t]{0.2\textwidth}
        \centering
        \includegraphics[width=0.9\linewidth]{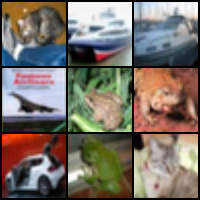}
        \caption{Input Image}
    \end{subfigure}%
    \begin{subfigure}[t]{0.2\textwidth}
        \centering
        \includegraphics[width=0.9\linewidth]{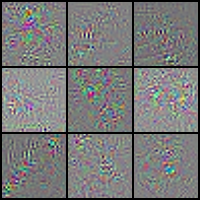}
        \caption{Baseline ResNet}
    \end{subfigure}%
    \begin{subfigure}[t]{0.2\textwidth}
        \centering
        \includegraphics[width=0.9\linewidth]{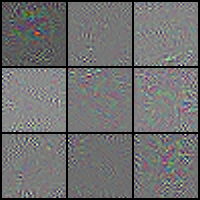}
        \caption{With Anti score-matching}
    \end{subfigure}%
    \begin{subfigure}[t]{0.2\textwidth}
        \centering
        \includegraphics[width=0.9\linewidth]{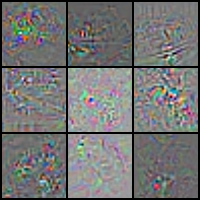}
        \caption{With Gradient-norm regularization}
    \end{subfigure}%
    \begin{subfigure}[t]{0.2\textwidth}
        \centering
        \includegraphics[width=0.9\linewidth]{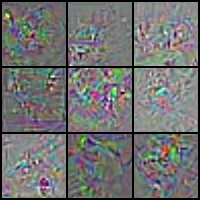}
        \caption{With Score-matching}
    \end{subfigure}
\caption{Visualization of input-gradients of different models. We observe that gradients of score-matched
and gradient-norm regularized models are more perceptually aligned than the others, with the gradients of the anti-score-matched model being the noisiest.
This qualitatively verifies the hypothesis stated in \S~\ref{sec:score-matching}.}
\end{figure}

\section{Hyper-parameter Sweep on Score-Matched Training}

We present results on a hyper-parameter sweep on the $\lambda$ and $\mu$ parameters of score-matching, where we 
provide both test-set accuracy on CIFAR100 and the corresponding GAN-test scores. We find upon performing a 
hyper-parameter sweep that $\lambda=1e-5$ and $\mu =1e-3$ seems to perform the best, whereas in the main paper 
we present results for $\lambda=1e-3$ and $\mu=1e-4$. It is possible that changing the training schedule by 
increasing the number of epochs or learning rate may further improve these results, but we did not explore that 
here. 

\begin{table}[h]
    \begin{center}
        \begin{tabular}{c | c  c  c  c }
            $\lambda \downarrow / \mu \rightarrow$ & $1e-2$ & $1e-3$ & $1e-4$ & $1e-5$ \\ 
            \midrule 
            $1e-2$ & $48.68\% / 51.60\%$  & $64.57\% / 58.90\%$ & $64.75\% / 76.46\%$ & $9.08 \% / 0.97\%$ \\
            $1e-3$ & $64.64\% / 56.78\%$ & $71.37\% / 40.72\%$ & $72.34\% / 73.39\%$ & $34.46\% / 3.3\%$ \\
            $1e-4$ & $69.85\% / 41.30\%$ & $73.97\% / 72.07\%$ & $75.65\% / 79.39\%$ & $72.97\% / 61.52\%$ \\
            $1e-5$ & $73.29\% / 68.94\%$ & $\mathbf{75.37\% / 85.64\%}$ & $76.40\% / 63.96\%$ & $74.80\% / 78.41\%$ \\ 
            $1e-6$ & $75.43\% / 82.81\%$ & $75.90\% / 66.11\%$ & $76.77\% / 65.91\%$ & $75.91\% / 65.52\%$ \\ 
        \end{tabular}
    \end{center}
    \caption{Results of a hyper-parameter sweep on $\lambda$ and $\mu$. The numbers presented are in the format (accuracy \% / GAN-test \% )}
\end{table}

\end{document}